\documentclass{article} 
\usepackage{iclr2022_conference,times}

\usepackage{amsmath,amsfonts,bm}

\def\eqref#1{equation~\ref{#1}}

\def\1{\bm{1}}

\DeclareMathAlphabet{\mathsfit}{\encodingdefault}{\sfdefault}{m}{sl}
\SetMathAlphabet{\mathsfit}{bold}{\encodingdefault}{\sfdefault}{bx}{n}

\usepackage{hyperref}
\usepackage{url}

\usepackage{graphicx}
\usepackage{amssymb}
\usepackage{lineno}
\usepackage{appendix}
\usepackage{url}
\usepackage{fancybox}
\usepackage{graphicx}
\usepackage{comment}

\usepackage{natbib}
\usepackage{color}
\definecolor{MyBrown}{rgb}{0.3,0,0}
\definecolor{MyBlue}{rgb}{0,0,1}
\definecolor{MyRed}{rgb}{0.5,0,0}
\definecolor{MyGreen}{rgb}{0,0.4,0}

\usepackage{algorithm}
\usepackage{algorithmic}
\usepackage{graphicx}  
\usepackage{dcolumn}   
\usepackage{bm}        
\usepackage{amsthm}
\usepackage{amsmath}
\usepackage{amssymb}
\usepackage{amsfonts}
\usepackage{amscd}
\usepackage{ascmac}
\usepackage{enumerate}
\usepackage{bm}
\usepackage{latexsym}
\usepackage{color}
\usepackage{soul}

\def\u{{\bf{u}}}

\def\u{{\bm u}}
\def\0{{\bm 0}}
\def\1{{\bm 1}}

\newtheorem{thm} {Theorem}
\newtheorem{lem}[thm] {Lemma}
\newtheorem{prp}[thm] {Proposition}
\newtheorem{df}[thm]{Definition}

\usepackage{algorithm}
\usepackage{algorithmic}
\usepackage[matrix,arrow]{xy}
\usepackage{subcaption}
\usepackage{booktabs}
\usepackage{wrapfig}

\newtheorem{conjecture}{Conjecture}
\allowdisplaybreaks[4] 

\title{Equivariant and Invariant Reynolds Networks}

\author{Akiyoshi Sannai \\
RIKEN AIP \\
Tokyo, Japan \\
\texttt{akiyoshi.sannai@riken.jp} \\
\And
Makoto Kawano \\
The University of Tokyo\\
Tokyo, Japan \\
\texttt{kawano@weblab.t.u-tokyo.ac.jp} \\
\And
Wataru Kumagai \\
The University of Tokyo, RIKEN AIP\\
Tokyo, Japan \\
\texttt{kumagai@weblab.t.u-tokyo.ac.jp} \\
}

\iclrfinalcopy 
\begin{document}

\maketitle
\begin{abstract}
Invariant and equivariant networks are useful in learning data with symmetry, including images, sets, point clouds, and graphs.
In this paper, we consider invariant and equivariant networks for symmetries of finite groups.
 Invariant and equivariant networks have been constructed by various researchers using Reynolds operators. 
However, Reynolds operators are computationally expensive when the order of the group is large because they use the sum over the whole group, which poses an implementation difficulty.
To overcome this difficulty, we consider representing the Reynolds operator as a sum over a subset instead of a sum over the whole group.
We call such a subset a Reynolds design, and an operator defined by a sum over a Reynolds design a reductive Reynolds operator.
For example, in the case of a graph with $n$ nodes, 
the computational complexity of the reductive Reynolds operator is reduced to $O(n^2)$, 
while the computational complexity of the Reynolds operator is $O(n!)$.
We construct learning models based on the reductive Reynolds operator called equivariant and invariant Reynolds networks (ReyNets) and prove that they have universal approximation property.
Reynolds designs for equivariant ReyNets are derived from combinatorial observations with Young diagrams, while Reynolds designs for invariant ReyNets are derived from invariants called Reynolds dimensions defined on the set of invariant polynomials.
Numerical experiments show that the performance of our models is comparable to state-of-the-art methods.

\end{abstract}

\section{Introduction}

The universal approximation theorem in machine learning states that any continuous function can be approximated by a deep neural network, but in a practical situation, there are many cases where learning fails.
It is believed that this is because the number of parameters used is too large to learn an appropriate model from a hypotheses set.
Therefore, researchers developed task-specific models, such as convolutional networks for image input, and obtained better results.
On the other hand, \citet{zaheer2017deep} developed a model using permutation invariance and permutation equivariance, and obtained good experimental results as well as theoretical development.
Such a model was generalized by  \citet{maron2019universality}  for subgroups of permutation groups.
In this paper, we develope an equivariant and invariant deep neural network model for the action of permutation groups on higher-order tensor spaces.
An important application of this formulation is the task of using graphs and hypergraphs as input.

The central part of our idea is to use the equivariant Reynolds operator. 
The equivariant Reynolds operator has the property of transforming a function into an equivariant function, and by adapting this operator to a deep neural nets, we obtain a class of equivariant neural networks. 
Such attempts have been made in the past, for example by \citet{yarotsky2021universal,kicki2021new,van2020mdp,mouli2020neural}, but the use of Reynolds operator causes computational problems when the order of the groups is large in the part that computes the equivariant Reynolds operator.

Therefore, we introduce the concept of Reynolds design, which is an analog of spherical design.
Here, a $t$-spherical design is a finite subset $H\subset \mathbb{S}^{d-1}$ such that any polynomial $f$ with degree at most $t$ satisfies
\begin{align}
    \frac{1}{|\mathbb{S}^{d-1}|}\int_{\mathbb{S}^{d-1}} f(x) dx
    = \frac{1}{|H|} \sum_{x \in H} f(x).
    \label{eq:spherical-design}
\end{align}
Using the spherical design,
the computation over the whole of $\mathbb{S}^{d-1}$ can be reduced to that over the small subset $H$.
Similarly, a Reynolds design is a subset $H$ of a group $G$ that realizes 
\begin{align*}
    \frac{1}{|G|}\sum_{g\in G} f(g\cdot x)
    = \frac{1}{|H|}\sum_{g\in H} f(g\cdot x)
\end{align*}
for a target function $f$, where $g\cdot x$ is the action of $g$ to $x$.
We call the operator defined on the right-hand side the reductive Reynolds operator, and use this operator instead of the Reynolds operator to construct equivariant and invariant networks. 
The advantage of this construction is that it preserves the universal approximation property of the construction using Reynolds operators.
Also, in many cases, the Reynolds design can be much smaller than the original $G$. 
To formulate this fact, we prove a representation theorem for equivariant functions using the relation between Young diagrams and the representation of symmetric groups on higher-order tensor spaces (Theorem \ref{represention}).
For example, in the case of graphs with $n$ nodes, the order of the group is $n!$ but we have shown that there exist Reynolds designs of order $n(n-1)$.
Thus, by using Reynolds designs, we can avoid the problem of computational complexity.

As for the invariant network, as in the case of CNN, the input is first transformed by a equivariant network, and after taking the sum of the orbits, it is transformed by a deep neural network.
Note that we do not know how to take the Reynolds design, because the above representation theorem is not proved with invariant functions.
To solve this problem we introduce an invariant called Reynolds dimension. Briefly, the Reynolds dimension is the number of variables required for the input of the functions before they are converted by the Reynolds operator, when the model had universality in the above construction. It is not clear what the Reynolds dimension is when taking a higher order tensor space as input, but our expectation is that it is independent of the dimension $n$. We prove that we can construct a Reynolds design when the Reynolds dimension is fixed. We also find that all input variables are not necessary when the Reynolds dimension is low, therefore we construct networks with reduced inputs (reduced ReyNets). 

Using these models, we conduct experiments on synthetic data.
In addition, as an implementation, we use the above representation theorem to develop the Corner MSE, which restricts the loss function for an equivariant function to two components.
As a result, our method significantly outperform the baseline of FNN and Maron's model, confirming its consistency with the theoretical part.

\section{Previous Work}

\textbf{Equivariance and Invariance}.
Various machine learning tasks aims to approximate a certain target map such as the labeling function in classification and regression.
When symmetries exist behind data, 
the target map often have invariance or equivariance to the symmetries.
In such cases, invariant or equivariant networks are effective and efficient to approximate the target map because the model complexity can be significantly reduced than neural networks without specific structure for the symmetries.
Convolutional neural networks (CNNs) are well-known as a seminal equivariant model to translation symmetry (\citet{lecun1989backpropagation}; \citet{krizhevsky2012imagenet}).
Inspired by the success of CNNs, 
various equivariant models have been proposed. 
%
%
Besides continuous symmetries such as translation, 
symmetries of finite groups often appear in many machine learning tasks.
In the case where sets or point clouds are inputs,
the target functions are typically invariant to the order of data points.
Then, this function has invariance to the permutation group on data points (\citet{qi2017pointnet,zaheer2017deep}).
In the case where graphs or hyper-graphs are inputs,
the symmetry is represented by permutation on a tensor product space.
\citet{gori2005new,scarselli2008graph} introduced graph neural networks. 
Various researchers generalized convolution to the setting of graphs motivated by CNNs (\citet{bruna2013spectral,henaff2015deep,kipf2016semi,defferrard2016convolutional,levie2018cayleynets}).
Recently, \citet{kondor2018covariant,maron2019provably,maron2019universality, chen2019equivalence} investigate more efficient graph neural networks.
%
\citet{ravanbakhsh2017equivariance} treated equivariant and invariant models to subgroups of the symmetric group.
\citet{hartford2018deep} consider interaction between sets.
\citet{graham2019deep} consider relational databases as generalization of graphs and provide equivariant models to handle relational databases.
\cite{maron2020learning} consider sets of symmetric elements as input data.

\textbf{Universality}.
The expressive power of learning models is mathematically validated by universal approximation theorems.
Many universal approximation theorems have been proved for different conditions.
Equivariant and invariant models with universal approximation property are provided for pointclouds networks and sets network (\citet{qi2017pointnet,zaheer2017deep}), 
tabular and multi-set networks (\citet{hartford2018deep}), graph and hyper-graph networks (\citet{kondor2018covariant,maron2018invariant,nguyen2020graph}), 
and networks invariant to finite translations with rotations and/or reflections \cite{maron2019universality}. 
On the other hand, \cite{oono2019graph} proved that graph neural networks of aggregation type are over smoothing when many layers are stacked.
There are also learning models with universality in other settings (\citet{yarotsky2021universal,keriven2019universal,maehara2019simple,segol2019universal}).

%

\section{Reynolds Operator and Reynolds Designs}


In this section, we organize the mathematical framework for dealing with symmetric networks. The first objects to be considered are the following equivariant and invariant functions.

\begin{df}[Invariant / Equivariant Function] For a group $G$ acting on $\mathbb{R}^{N}$ and $\mathbb{R}^{M}$, a function 
$f: \mathbb{R}^{N} \rightarrow \mathbb{R}^{M}$ is invariant if $f(g \cdot x)=f(x)$ holds for any $g \in G$ and any $x \in \mathbb{R}^{N}$, and  equivariant if $f(g \cdot x)=g \cdot f(x)$ holds for any $g \in G$ and any $x \in \mathbb{R}^{N}$.
\end{df}

The next operator, called the Reynolds operator, plays a central role in this paper.

\begin{df}[Reynolds Operator (cf. \cite{mumford1994geometric}, Definition 1.5)]
For a group $G$,
the followings are called the equivariant and invariant Reynolds operator respectively:
\begin{align}
\tau_{G}(f(-))=\frac{1}{|G|} \sum_{g\in G} g^{-1} \cdot f(g \cdot-), 
\quad\gamma_{G}(f(-))=\frac{1}{|G|} \sum_{g\in G} f(g \cdot-).
\label{eq:equiv}
\end{align}
\end{df}

More generally,
for a subset $H$ in $G$\footnote{The subset $H$ in $G$ may not be a subgroup of $G$.},
we define $\tau_H$ and $\gamma_H$ by replacing $G$ by $H$ in (\ref{eq:equiv}). 

The equivariant Reynolds operator converts an arbitrary map to a equivariant map. Our idea is to use this operator to convert a fully connected deep neural network into an equivariant function. 
However, as we see in a moment, the computational complexity of Reynolds operators increases when the order of $G$ is large, for example, for symmetric groups of order $n!$.


Design theory in mathematics provides a suitable subset that represents some property about the whole.
Here, we explain spherical design.
Let $\mathcal{P}_t$ be the set of all polynomials of at most degree $t\in\mathbb{N}$ on $\mathbb{R}^d$. 
Then, there exists a finite subset $H$ of the sphere $\mathbb{S}^{d-1}\subset \mathbb{R}^d$ such that an arbitrary polynomial $p\in \mathcal{P}_t$ satisfies (\ref{eq:spherical-design}).
Then, such a subset $H$ is called the $t$-spherical design.

Therefore, we define the following notion of Reynolds design to reduce the computational complexity. 
The concept of Reynolds design is first proposed in this paper.

\begin{df}[Reynolds Design]
The Reynolds design $H$ of a function $f$ is a subset $H$ of $G$ that satisfies the following equation.

\begin{align*}
\tau_G(f(-))=\frac{1}{|G|} \sum_{g \in G} g^{-1} \cdot f(g \cdot-)=\frac{1}{|H|} \sum_{g \in H} g^{-1} \cdot f(g \cdot-).
\end{align*}
In this case, $\tau_H$ is called the reductive Reynolds operator.Furthermore, the Reynolds design $H$ of a set of functions $\mathcal{F}$ is defined by the property that the above equation for $H$ holds for all $f \in \mathcal{F}$.

\end{df}

As an example, consider a power of one variable as $f$. Since $\tau_{S_{n}}\left(x_{1}^{i}\right)=x_{1}^{i}+\cdots+x_{n}^{i}$, the number of terms in the summation is only $n$, which can be written as an orbit by the cyclic group $C_n$ of order $n$ as 
$\tau_{S_{n}}\left(x_{1}^{i}\right)=x_{1}^{i}+\cdots+x_{n}^{i}=\tau_{C_{n}}\left(x_{1}^{i}\right)$. 
Hence $C_n$ is the Reynolds design of $x_1^i$.

\section{Representation Theorem for Equivariant Maps}




In this section, we consider the case where the symmetric group $S_n$ acts on the higher order tensor space $\mathbb{R}^{n^{\ell} \times a}$. 
The action of $S_n$ on tensors $\mathbf{X} \in \mathbb{R}^{n^{l} \times a}$ (the last index, denoted $\alpha$ represents feature depth) is defined by  $(g \cdot \mathbf{X})_{i_{1} \ldots i_{l}, \alpha}=\mathbf{X}_{g^{-1}\left(i_{1}\right) \ldots g^{-1}\left(i_{l}\right), \alpha}$.
Note that this is an important class for applications because it deals with the case where the input is a set, graph, etc. 
\begin{figure*}[!t]
\centering
    \includegraphics[bb=0 0 1150 450, 
        scale=0.25]{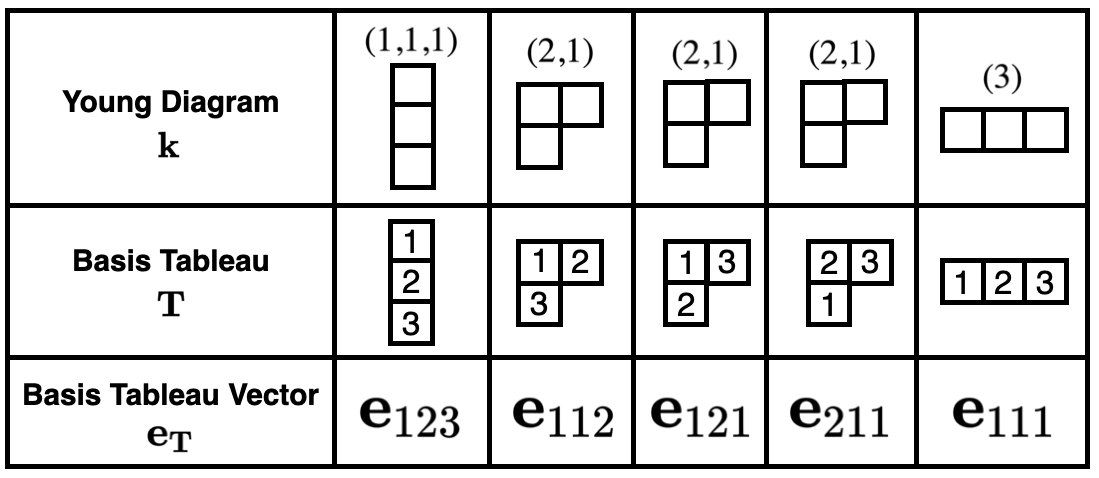}
\caption{
Young diagrams and basis tableaux for $m=3$, and the corresponding basis tableau vectors in $\mathbb{R}^{n^3}= \mathbb{R}^{n}\otimes\mathbb{R}^{n}\otimes\mathbb{R}^{n}$.
}
\label{fig:equivariance}
\end{figure*}

A Young diagram is a way to represent the division of a natural number $m$. 
Here, division means to express a natural number $m$ as a sum of several non-negative integers.

\begin{df}[Young Diagram]
Let $D\in[m]$ be fixed. A vector $\mathbf{k}=[k_1,\ldots,k_D]$ of natural numbers $k_1,\ldots,k_D$ is called a Young diagram if it satisfies $m = k_1+\ldots+k_D$ and $k_1\ge \ldots\ge k_D$.
\end{df}

This partition can be represented by a total of $m$ boxes, consisting of $D$ rows with $k_i$ boxes in the $d$-th row. 
Here, each row is left-justified.
We formally provide the definition of basis tableaux in order to use the notation in later discussion.

\begin{df}[Basis Tableau]
Let $n\ge m$  and $\mathbf{k}=[k_1,\ldots,k_D]$ a Young diagram. 
A vector $\mathbf{T}= \left[\mathbf{t}_1, \ldots, \mathbf{t}_D \right]$ of vectors $\mathbf{t}_d= [t_{d,1}, \ldots, t_{d,k_d}]\in [n]^{k_d}$ is called basis tableau of depth $D$ if it satisfies the following conditions:
\begin{enumerate}
    \item (Different components) $t_{d,w} \ne t_{d',w'}$ for $(d,w)\ne (d',w')$.
    \item (Row monotonicity) $t_{d,1}<t_{d,2}<\ldots< t_{d,k_d}$ for each $d\in[D]$. 
    \item (Partial column monotonicity) If $k_d= k_{d+1}$, then $t_{d,1}<t_{d+1,1}$ for $d=1,\ldots, D-1$.
\end{enumerate}
\end{df}

We denote the set of basis tableaux for $m$ of depth $D$ by $\mathcal{T}_{m,D}$,
and 
the set $\bigcup_{1\le D \le m} \mathcal{T}_{m,D}$ of basis tableaux with at most depth $m$ by $\mathcal{T}_m$.

Here, we introduce some notation about basis to state Theorem \ref{represention}.
Let $\mathbf{e}_1, \ldots, \mathbf{e}_n$ the standard basis of $\mathbb{R}^n$.
For $\mathbf{u}=[u_1,\ldots,u_m] \in [n]^m$,
we set 
\begin{align}
    \mathbf{e}_{\mathbf{u}}:=\mathbf{e}_{u_1,\ldots,u_m}=\mathbf{e}_{u_1}\otimes \cdots\otimes\mathbf{e}_{u_m} \in \underbrace{\mathbb{R}^n\otimes \cdots \otimes \mathbb{R}^n}_{m}=\mathbb{R}^{n^m}.
    \label{eq:basis}
\end{align}
For any basis tableau $\textbf{T}\in \mathcal{T}_{m,D}$,
natural numbers $\mathbf{u}=[u_1,\ldots, u_m]\in [n]^m$ are given as
$
    u_{\ell}:= d\in [n] \quad \textit{if} \quad \ell \in \{\mathbf{t}_{d}\}. 
$\footnote{
For a vector $\mathbf{t}=[t_{1},\ldots,t_{k}]$,
we set $\{\mathbf{t}\}:= \{t_{1},\ldots,t_{k}\}$. 
For example, when $\mathbf{t}=[1,1]$, $\{\mathbf{t}\}=\{1\}$.
}
Then, we define the map $\phi: \mathcal{T}_m \to [n]^m$ by $\phi(\textbf{T}):= (u_1,..,u_m) \in [n]^m$. 

\begin{df}[Basis Tableau Vector]
Let $n\ge m$. 
For $D\in[m]$
and
$\mathbf{T}= \left[\mathbf{t}_1, \ldots, \mathbf{t}_D \right]\in \mathcal{T}_{m,D}$, 
the basis tableau vector is defined by $\mathbf{e}_{\mathbf{T}}:= \mathbf{e}_{\phi(\mathbf{T})}\in \mathbb{R}^{n^m}$, where the right hand side is defined by (\ref{eq:basis}).
\end{df}
Figure \ref{fig:equivariance} shows an example of Young diagrams, basis tableaux, and basis tableaux vectors when $m = 3$. 
For $\mathbf{X} \in \mathbb{R}^{n^m}$,
we define the linear map by
\begin{align}
    \hat{\mathbf{X}}_{b}&:\mathbb{R}^b\ni (a_1, \ldots, a_b) \mapsto  (a_1\mathbf{X}, \ldots, a_b\mathbf{X}) \in \mathbb{R}^{n^m\times b}.
    \label{eq:linear}
\end{align}
In the following, 
we mainly treat the case where 
$\mathbf{X}=\mathbf{e}_{\mathbf{T}}$ (i.e., $\hat{\mathbf{X}}_{b}=\hat{\mathbf{e}}_{\mathbf{T},b}$).
Then, we have the following theorem.
\begin{thm}[Representation Theorem]\label{represention}
Let $n\ge m$ and $G=S_n$.
For any equivariant continuous map  $F: \mathbb{R}^{n^l \times a} \rightarrow \mathbb{R}^{n^m \times b}$, there exist continuous maps $F_{\mathbf{T}}: \mathbb{R}^{n^l \times a} \rightarrow \mathbb{R}^b$ indexed by basis tableaux $\mathbf{T} \in \mathcal{T}_m$ such that 
\begin{align*}
F= \sum_{D=1}^m \sum_{\mathbf{T} \in \mathcal{T}_{m,D}}
 \tau_{H_D}(F_{\textbf{T}}\circ \hat{\mathbf{e}}_{\mathbf{T}, b}),
  \end{align*}where $C_{n-i}$ denote the cyclic group of order $n-i$ on the set $\{i+1,..,n\}$ and
  \begin{align*}
      H_{D} &:= C_n \circ \cdots \circ C_{n-D+1}
      = \{\sigma_D \cdot \sigma_{D-1} \cdots \sigma_{1}  \mid \sigma_i \in C_{n-D+i}~ (i=n,n-1,\ldots,n-D+1)\}.
  \end{align*}
Furthermore,  $ H_{D} $ is a Reynolds design of $F_{\textbf{T}
}\circ \hat{\mathbf{e}}_{\mathbf{T}, b}$ for any $\textbf{T}\in \mathcal{T}_{m,D}$.
\end{thm}

\subsection{Proof of Theorem \ref{represention}}


We introduce the following notion to obtain tableau-based representation of elements in $[n]^m$.

\begin{df}[Extended Tableau]
Let $n\ge D$.
Let $[n]^D_{\#}$ be the set of $D$ different natural numbers at most $n$ defined by
$
    [n]^D_{\#}:= \{[j_1, \ldots, j_D] \in [n]^D \mid j_a \ne j_b ~\text{for}~ a\ne b \in [D]\}.
$
Then,
we call elements in $\tilde{\mathcal{T}}_{m,D}:=[n]^D_{\#}\times \mathcal{T}_{m,D}$ extended tableaux with depth $D$.
\end{df}

We denote the set $\bigcup_{1\le D \le m} \tilde{\mathcal{T}}_{m,D}$ of extended tableaux with at most depth $m$  by $\tilde{\mathcal{T}}_m$.
The action $g\cdot [j_1, \ldots, j_D]:=[g\cdot j_1, \ldots, g\cdot j_D]$ of $G$ on $[n]^D_{\#}$ is well-defined.
Then, we define the action of $G$ on $\tilde{\mathcal{T}}_{m,D}$ by $g\cdot (\mathbf{j}, \mathbf{T}) := (g\cdot \mathbf{j}, \mathbf{T})$.
In the following, 
we identify an extended tableau $([1,2,\ldots, D],\mathbf{T})$ with the basis tableau $\mathbf{T}$.

Next, we introduce some notation to represent elements in $[n]^m$ by extended tableaux.
We first define the partial order $\succ$ of vectors of natural numbers that can have different dimensions.
For $\mathbf{t} =[t_1,\ldots,t_k]\in \mathbb{N}^k$ and $\mathbf{t}' =[t'_1,\ldots,t'_k]\in \mathbb{N}^{k'}$,
we denote as
$\mathbf{t} \succ \mathbf{t}'$ if either (i) $k>k'$ or (ii) $k=k'$ and $t_1>t'_1$.
We note that $\mathbf{t}_1\succ \ldots \succ \mathbf{t}_D$ holds for a basis tableau $\mathbf{T}=[\mathbf{t}_1,\ldots,\mathbf{t}_D]$ by definition.
Let $\mathbf{u}=[u_1,\ldots,u_k]\in [n]^{m}$.
We set $D(\mathbf{u}):=|\{\mathbf{u}\}|$.
We define the multiplicity map $\mathrm{mult}_{\mathbf{u}}:\{\mathbf{u}\} \to [m]$ by  
$
    \mathrm{mult}_{\mathbf{u}}(u)
    := |\{\ell \in [m] \mid u_\ell = u\}|.
$
For $u\in \{\mathbf{u}\}$,
we define $\mathbf{t}_u:= [t_1,\ldots,t_{k_u}]$, 
where $k_u:=\mathrm{mult}_{\mathbf{u}}(u)$,
$u_{t_1}=\cdots=\u_{t_{k_u}}=\mathbf{u}$,
and $t_{1}<\cdots<t_{k_u}$.
Here, we note that either $\mathbf{t}_{u}\succ \mathbf{t}_{u'}$ or $\mathbf{t}_{u}\prec \mathbf{t}_{u'}$ holds for $u\ne u' \in \{\mathbf{u}\}$ by definition.
Thus, 
there exist different natural numbers $j_1,\ldots,j_D \subset [n]$ such that $\{j_1,\ldots,j_D\} = \{\mathbf{u}\}$ and $\mathbf{t}_{j_1} \succ \ldots \succ \mathbf{t}_{j_D}$.  

\begin{df}[Tableau Representation]
Let $n\ge m$ .
We define the map $\psi: [n]^m \ni \mathbf{u} \mapsto (\mathbf{j},\mathbf{T}) \in \tilde{\mathcal{T}}_m$ called tableau representation by
\begin{align*}
    \psi(\mathbf{u}) 
    := ([j_1, \ldots, j_D], [\mathbf{t}_{j_1}, \ldots, \mathbf{t}_{j_D}])
    \in \tilde{\mathcal{T}}_{m,D}.
\end{align*}
\end{df}

We define the map $\phi: \tilde{\mathcal{T}}_m \to [n]^m$ as follows:
For any extended tableau $(\mathbf{j}, \textbf{T})\in \tilde{\mathcal{T}}_{m,D}$,
natural numbers $\mathbf{u}=[u_1,\ldots, u_m]\in [n]^m$ are given as
$
    u_{\ell}:= j_d\in [n] \quad \textit{if} \quad \ell \in \{\mathbf{t}_{d}\}, 
$
and $\phi(\mathbf{j}, \textbf{T}):= (u_1,..,u_m) \in [n]^m$.

\begin{lem}\label{lem:bijective1}
    When $n\ge m$,
    the tableau representation $\psi: [n]^m \to \tilde{\mathcal{T}}_m$ is bijective and $\psi(g\cdot \mathbf{u}) = g\cdot \psi(\mathbf{u})$ for $g \in H_D$.
\end{lem}
\begin{proof}
First, from the construction, $\psi$ is injective and $\phi\circ \psi = \mathrm{id}_{[n]^m}$.
Since $[n]^m$ and $\mathcal{T}$ are finite sets, if we show that $\phi$ is injective, then $\psi$ is bijective.
For the extended tableau $\mathbf{S},\mathbf{T}$, assume that $\phi(\mathbf{S})=\phi(\mathbf{T})$. 
Note that the set of row vectors $\{\mathbf{s}_1,...,\mathbf{s}_d\}$ of $\mathbf{S}$ is uniquely determined from $\phi(\mathbf{S})$. 
Then, since the extended tableaux satisfy the order $\succ$ in which this goes between the row vectors, the extended tableau $\textbf{S}$ having row vectors $\{\mathbf{s}_1,...,\mathbf{s}_d\}$ is unique. Hence, we have $\textbf{S}=\textbf{T}$.
\end{proof}

\begin{df}[Extended Tableau Vector]
Let $n\ge m$.
The extended tableau vector $\mathbf{e}_{\mathbf{j},\mathbf{T}}$ is defined by $\mathbf{e}_{\mathbf{j},\mathbf{T}}:=\mathbf{e}_{\phi(\mathbf{j},\mathbf{T})}$.
\end{df}
%
%
For an extended tableau $(\mathbf{j},\mathbf{T})\in \tilde{\mathcal{T}}_m$,
the linear map $\hat{\mathbf{e}}_{\mathbf{j},\mathbf{T},b}: \mathbb{R}^b \rightarrow \mathbb{R}^{n^m\times b}$
is defined by $\mathbf{X}=\mathbf{e}_{\mathbf{j},\mathbf{T}}$ in
(\ref{eq:linear}).

\begin{lem}[Normalization]\label{lem:bijective2}
    Let $n\ge D$.
    For $\mathbf{j}\in[n]^D_{\#}$, there uniquely exists $g\in H_D$ such that $\mathbf{j}= g^{-1}\cdot [1,2,\ldots, D]\in[n]^D_{\#}$.
    Hence, this correspondence $[n]^D_{\#} \ni \mathbf{j} \to g \in H_D$ is bijective.
\end{lem}

\begin{proof}
There uniquely exists $\sigma_{1} \in C_{n}$ such that $\sigma_{1}\left(j_{1}\right)=1$.
Inductively, there uniquely exists $\sigma_{d} \in C_{n-d+1}$ such that $\sigma_{d}\left(\sigma_{d-1}\cdot \sigma_{d-2}  \cdots  \sigma_1 ( j_{d})\right)=d$ for $d=2,\ldots,D$.
Then $g := \sigma_{D} \cdots \sigma_{1}\in H_D$ satisfies $g\cdot \mathbf{j}=  [1,2,\ldots, D]$ by definition. 
\end{proof}

From Lemma \ref{lem:bijective2},
an extended tableau $(\mathbf{j},\mathbf{T})\in \tilde{\mathcal{T}}_{m,D}$ is uniquely represented by 
$
    (\mathbf{j},\mathbf{T}) = g^{-1}\cdot([1,\ldots,D], \mathbf{T})
    =g^{-1}\cdot\mathbf{T}
$
as in Figure \ref{fig:reduction},
where $g\in H_D$ and we identified $\mathbf{T}\in \mathcal{T}_{m,D}$ with $([1,\ldots,D], \mathbf{T})\in \tilde{\mathcal{T}}_{m,D}$ in the last equation.
Thus, we have $\tilde{\mathcal{T}}_{m,D}= \bigcup_{g\in H_D} g^{-1} \mathcal{T}_{m,D}$, 
and
$\tilde{\mathcal{T}}_{m}= \bigcup_{1\le D \le m} \bigcup_{g\in H_D} g^{-1} \mathcal{T}_{m,D}$.
From Lemma \ref{lem:bijective1}, for each $\mathbf{u} \in [n]^m$, there uniquely exists $g\in H_D$ and $\mathbf{T}\in\mathcal{T}_{m}$ such that $\psi (\mathbf{u}) = g^{-1}\cdot \mathbf{T}$ (or equivalently $\mathbf{u} = \psi^{-1}(g^{-1}\cdot \mathbf{T})$). 
In the following, we omit the bijective $\psi$ for notational simplicity.



In the following, 
we prove Theorem \ref{represention}.
We can write $F =  \sum_{\mathbf{u} \in\left[n\right]^{m}}  f_{\mathbf{u}} \cdot \hat{e}_{\mathbf{u}, b}$ by maps $f_{\mathbf{u}}: \mathbb{R}^{n^l \times a} \rightarrow \mathbb{R}^b$. 
Then since $F$ is equivariant, we have
$
\sum_{\mathbf{u} \in\left[n\right]^{m}} f_{\mathbf{u}}\cdot \hat{e}_{\mathbf{u}, b}(g\cdot x) 
=F(g \cdot x)
=g\cdot F(x)
=\sum_{\mathbf{u} \in\left[n\right]^{m}} f_{\mathbf{u}} \cdot \hat{e}_{g\cdot \mathbf{u}, b}(x).
$
This implies that 
\begin{align}
    f_{\mathbf{u}} \cdot \hat{e}_{\mathbf{u}, b}(g\cdot x)
=f_{g^{-1}\cdot \mathbf{u}}\cdot \hat{e}_{\mathbf{u}, b}(x). \label{eq:component}
\end{align}

\begin{figure*}[!t]
\centering
    \includegraphics[bb=0 0 1590 300, 
    scale=0.24]{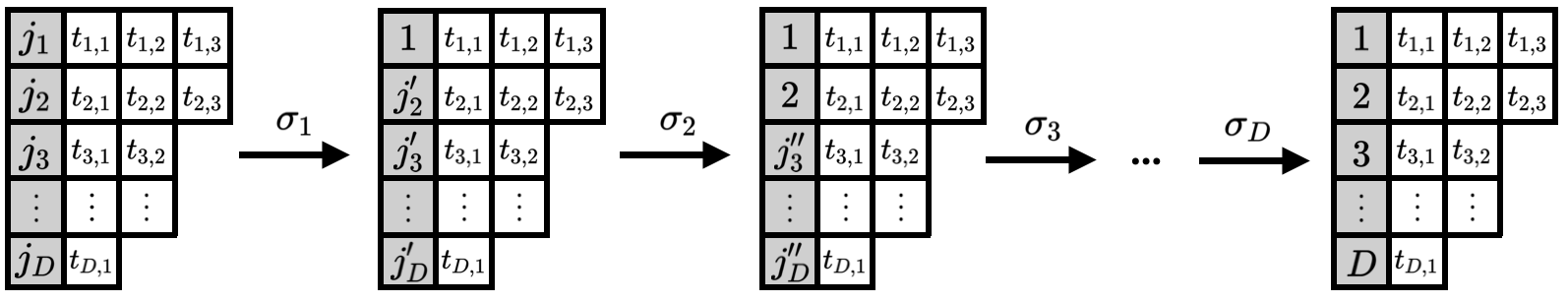}
\caption{
Reduction from an extended tableau to a basis tableau by permutations.
An extended tableau $(\mathbf{j},\mathbf{T})$ is converted to a basis tableau by multiplying $g=\sigma_D \cdot \sigma_{D-1}\cdots \sigma_{1}$.
}
\label{fig:reduction}
\end{figure*}

Then,
we obtain the following equations:
\begin{align}
F(x) 
=\sum_{1\leq D\leq m}\sum_{\textbf{T}  \in \mathcal{T}_{m, D}} \tau_{H_{D}}\left( F_{\textbf{T} } \cdot \hat{e}_{\textbf{T} , b}\right)(x).
\label{eq:representation}
\end{align}
We provide the detailed calculation of (\ref{eq:representation}) in the appendix.


\section{Equivariant ReyNets and Universality}

In this section, we define equivariant models using Theorem \ref{represention}. 

\begin{df}[Equivariant Reynolds Nets]
We assume that $n\ge m$.
For any basis tableaux $\mathbf{T} \in \mathcal{T}_m$, Let $\mathcal{N}_{\mathbf{T}}: \mathbb{R}^{n^l \times a} \rightarrow \mathbb{R}^b$ be a Multi layer perceptron. 
The map $\mathcal{E}: \mathbb{R}^{n^l \times a} \rightarrow \mathbb{R}^{n^m \times b}$
\begin{align*}
\mathcal{E}= \sum_{D=1}^m \sum_{\mathbf{T} \in \mathcal{T}_{m,D}} \tau_{H_D}(\mathcal{N}_{\textbf{T}
}\circ \hat{\mathbf{e}}_{\mathbf{T}, b}),
  \end{align*} is called reduced Equivariant Reynolds Nets (equivariant ReyNets). 
\end{df}

Theorem \ref{represention} guarantees the universal approximation property of this model.
\begin{thm}[Universality]\label{thm:universality-equiv}
    We assume that $n\ge m$ and $G=S_n$.
  Let $F:  \mathbb{R}^{n^l \times a} \rightarrow \mathbb{R}^{n^m \times b}$ be a continuous equivariant function. For any compact set $K \subset \mathbb{R}^{n^{\ell}\times a}$ , there exists a equivariant Reynolds network that approximates $F$ to an arbitrary precision on $K$.
Namely, equivariant Reynolds nets are a universal approximator for equivariant functions.
\end{thm}

The most important application, the case of the graph, corresponds to the $m = 2$ case. In this case, the Young diagrams are $\textbf{k}=(2)$ and $(1,1)$, and the basis tableaux is also uniquely numbered in the above two. In other words, the number of MLP $\mathcal{N}_T$ required to construct the model is two. For the Reynolds design, $H_1$ is of order $n$ and $H_2$ is of order $n(n-1)$, so the total computational complexity is $O(n^2)$. The number of summations that were originally required $n!$ times has been reduced to $n^2$ times, resulting in a significant reduction in the amount of calculation. We can see some experiments on these in the experiments section.

\section{Invariant ReyNets and Universality}

In this section, we define the invariant models using the equivariant models defined in the previous section. First, we will define invariant Reynolds nets. Then, we review the generators of invariant polynomials and define Reynolds dimension using them. Finally, we explain the universal approximation of invariant networks using Reynolds dimension. We will deal with the case where a general finite group $G$ acts on an $n$-dimensional vector space or a set of $n$ variables. Note that such formulations include the case of higher order tensor space as $G=S_n$.
\begin{df}[Invariant Reynolds Nets]
A invariant Reynolds network (invariant ReyNet) is a function $\mathcal{I}$ :
$\mathbb{R}^{n^l \times a} \rightarrow \mathbb{R}$ defined as
$$
\mathcal{I} = \mathcal{M} \circ \Sigma \circ \mathcal{E},
$$
where $\mathcal{E}: \mathbb{R}^{n^l \times a} \rightarrow \mathbb{R}^{n^m \times b}$ is a equivariant Reynolds network, 
$\Sigma$ is the orbit sum\footnote{
The orbit sum $\Sigma: \mathbb{R}^{[n]^m\times b} \to \mathbb{R}^{[n]^m/G\times b}$ is defined by $\Sigma(\mathbf{X})_{G\cdot \mathbf{u},\beta}:= \sum_{g\in G} x_{g\cdot \mathbf{u},\beta}$ for $\mathbf{X}=[x_{\mathbf{u},\beta}] \in \mathbb{R}^{[n]^m\times b}$, $\mathbf{u}\in [n]^m$, $\beta \in [b]$ and $G\cdot \mathbf{u} \in [n]^m/G$.}, 
and $\mathcal{M}$ is a Multi-Layer Perceptron (MLP).
\end{df}

We define the notion of generator of invariant polynomials, which is necessary to define Reynolds dimension.
\begin{df}[Generator of invariant polynomials]
Let $G$ be a finite group.
A set of invariant polynomials $r_1(x_1,..,x_n),..,r_s(x_1,..,x_n)$ are called {\it a generator of invariant polynomials} if for any invariant polynomial $f(x_1,..,x_n)$, there is a polynomial $p(y_1,..,y_s)$ such that $f(x_1,..,x_n)=p( r_1(x_1,..,x_n),..,r_s(x_1,..,x_n))$ holds.
\end{df}

The existence of such a generator is non-trivial for general group actions, but has been shown by Hilbert for finite groups or more generally linearly reductive groups \cite{hilbert1890}.
Under the preparation so far, we define the following Reynolds dimension.
\begin{df}[Reynolds dimension]
Let $G$ be a finite group.
The smallest natural number $d$ satisfying the following property is the Reynolds dimension of the group $G$. 
There exist polynomial $h_1,..,h_s$ of $d$-variable and an index subset $\{j_1,\ldots,j_d\}\subset [n]$  such that$\gamma_{G}(h_{1}(x_{j_1}, \ldots, x_{j_d})), \ldots, \gamma_{G}(h_{s}(x_{j_1}, \ldots, x_{j_d}))$ is a generator of the invariant polynomials.
\end{df}


Before discussing universality for invariant functions, let us review some notation. We define $\mathrm{Stab}_G([d])$ to be the set of elements of $G$ for which $x_{j_1}, \ldots, x_{j_d}$ are fixed. 
In addition, $[G/G']$ a complete system of representatives of $G/G'$ for a subgroup $G'\subset G$ is a set of order $|G/G'|$ that satisfies $G = \bigcup_{a \in [G/G']} a G'$. 

\begin{prp}\label{inv design}
In the same situation as Definition 18, $[G/\mathrm{Stab}_G([d])]$ is a Reynolds design of  $h_i$.
\end{prp}
\begin{thm}[Universality]\label{thm:universality-inv}
We assume that $G=S_n$.
Let $d$ be the Reynolds dimension of G.
Then, Reyonlds invariant nets constructed above for 
$\mathcal{E}: \mathbb{R}^{n^l \times a} \rightarrow \mathbb{R}^{n^d \times b}$ is a universal approximator for invariant functions $f: \mathbb{R}^{n^{\ell} \times a} \rightarrow \mathbb{R}^b$. 
More strongly, the input space of $\mathcal{E}$  can be replaced by a composite $ \mathcal{E}\circ \mathcal{Z}$ with the the zero padding map\footnote{
The zero padding map $\mathcal{Z}: \mathbb{R}^{d\times a} = \mathbb{R}^d \otimes \mathbb{R}^a \to \mathbb{R}^{n^{\ell}\times a} = \mathbb{R}^{n^{\ell}}\otimes \mathbb{R}^a$ is the linear map defined by 
$
    \mathcal{Z}((x_1, \ldots, x_d) \otimes \mathbf{e}_{\alpha}) 
    := (x_1, \ldots, x_d, \underbrace{0,\ldots, 0}_{n^\ell -d}) \otimes \mathbf{e}_{\alpha}
$
for $\alpha=1,2,\ldots,a$.
} $\mathcal{Z}: \mathbb{R}^{d\times a} \to \mathbb{R}^{n^{\ell}\times a}$.
\end{thm}

\section{Reduction of input variables}
So far, we have taken the input of $\mathcal{N}$ as an $n^{\ell}$-dimensional space. However, when $n$ or $\ell$ are large, this can be an obstacle to the calculation. Therefore, we take $d$-dimensional subspace as an input by Theorem \ref{thm:universality-inv}. This is called the $d$-reduced equivariant (invariant) Reynolds network. For example, consider the case $\ell=2$ and assume that the set of invariant polynomials $\mathbb{R}[\textbf{x}]^{S_n}$ has Reynolds dimension $4$ by the valuables $x_{1,1},x_{1,2},x_{2,1}, x_{2,2}$, then  by Theorem \ref{thm:universality-inv}, we can restrict the input space of $f_\textbf{T}$ in our model to $x_{1,1},x_{1,2},x_{2,1}, x_{2,2}$ and still preserve universality for invariant functions. 
The reduced Reynolds networks give us the great benefit of generalization over $n$. 
In other words, by transferring $f_\textbf{T}$ learned for small $n$ to $f_\textbf{T}$ for large $n$, we can obtain a model for large $n$. See Figure 3 for more details.
When we consider a practical task, for example, the problem of classifying graphs, it is not a simple matter to determine how many Reynolds dimensions it has. However, this problem has been formulated mathematically as a problem in higher-order tensor spaces, and we expect that the Reynolds dimension will be independent of $n$ in this case.
\begin{conjecture}
Let $G$ be an $S_n$ acting on $\mathbb{R}^{n^{\ell}}$ as in Section 4. For arbitrary enough large $n$, the Reynolds dimension $d(n)$ of G is independent of $n$.
\end{conjecture}

\section{Experiments}

\begin{table}[]
\centering
\caption{Results of comparison to a baseline method}
\fontsize{8pt}{9pt}\selectfont
\begin{tabular}{@{}lllllllll@{}}
\toprule
Task                & \multicolumn{4}{l}{Symmetry}                                                  & \multicolumn{4}{l}{Diagonal}                                                  \\
\textit{\textbf{n}} & 3                 & 5                 & 10                & 20                & 3                 & 5                 & 10                & 20                \\ \midrule
FNN               & 1.730e-4          & 9.180e-4          & 1.454e-3          & 3.0583          & 1.295e-4          & 2.655e-4          & 1.148e-4          & 1.081e-1          \\
\citet{maron2018invariant}               & 6.600e-3          & 3.786e-3          & 9.294e-4          & 4.471e-3          & 2.065e-3          & 2.266e-3          & 4.098e-3          & 4.743e-4          \\
ReyNet (ours)     & 2.147e-4          & 3.960e-4          & 1.408e-3          & 3.151e-3          & 1.007e-4          & 2.472e-4          & 6.635e-4          & 1.112e-4          \\    
4red-ReyNet (ours)         & \textbf{8.544e-5} & \textbf{4.889e-5} & \textbf{7.529e-5} & \textbf{6.554e-5} & \textbf{6.947e-5} & \textbf{1.932e-5} & \textbf{5.568e-5} & \textbf{3.566e-5} \\ \bottomrule
\toprule
Task                & \multicolumn{4}{l}{Power}                                                     & \multicolumn{4}{l}{Trace}                                                     \\
\textit{\textbf{n}} & 3                 & 5                 & 10                & 20                & 3                 & 5                 & 10                & 20                \\ \midrule
FNN               & 2.586          & 3.091e+1          &    5.756e+1       & 6.268e+2          & 2.821e-4          & 1.135e-3          & 9.529e-3          & 5.149e-2          \\
\citet{maron2018invariant}               & 4.036e-1          & 3.462e-1          & 7.062e-1          & 4.735e-1          & 1.241e-3          & 6.696e-3          & 3.663e-2          & 5.527e-2          \\
ReyNet (ours)     & 3.798e-1          & 1.257             & 3.620             & 3.065             & 1.884e-3          & 2.949e-3          & 4.275e-2          & 5.338e-2          \\
4red-ReyNet (ours)         & \textbf{1.204e-1} & \textbf{1.330e-1} & \textbf{1.217e-1} & \textbf{1.165e-1} & \textbf{3.491e-4} & \textbf{1.914e-3} & \textbf{6.758e-3} & \textbf{1.220e-2} \\ \bottomrule
\end{tabular}
\label{tab:exp1}
\end{table}

We evaluated our models compared to fully-connected neural nets (FNNs) and \citet{maron2018invariant} using synthetic datasets for equivariant and invariant tasks.
In the experiments, we used two models; Reynolds Networks (ReyNets) and $4$-reduced Reynolds Networks (4red-ReyNets).

\subsection{Datasets}
We created four types of synthetic datasets for the comparison.
Given an input matrics data $\mathbf{A} \in \mathbb{R}^{n \times n}$, each task is defined by:
\begin{itemize}
    \item \textbf{Symmetry}: projection onto the symmetric matrices $F(\mathbf{A})=\frac{1}{2}(\mathbf{A} + \mathbf{A}^\top)$,
    \item \textbf{Diagonal}: diagonal extraction $F(\mathbf{A})=diag(\mathbf{A})$,
    \item \textbf{Power}: computing each squared element $F(\mathbf{A})=[A_{i,j}^2]$,
    \item \textbf{Trace}: computing the trace $F(\mathbf{A})=tr(\mathbf{A})$,
\end{itemize}
where the task function $F$ is equivariant in \textit{symmetry}, \textit{diagonal}, and \textit{power}, and is invariant in \textit{trace}.
Following \citet{maron2018invariant}, we sampled i.i.d. random matrices $\mathbf{A}$ with uniform distribution in $[0, 10]$; then we transformed it by the task function $F$.
In our experiments, we provided $n \in \{ 3, 5, 10, 20\}$, and the size of training dataset and test dataset is 1000 respectively.

\subsection{Implementation Details}
For equivariant tasks, we applied equivariant ReyNets.
For invariant tasks, we applied invariant ReyNets following the architecture of the invariant model proposed by \citet{maron2018invariant};
an equivariant ReyNet is followed by max pooling\footnote{This operation outputs the max value of diagonal and non-diagonal elements of an input matrix.}, and fully-connected layers.
We adopted the ReLU function as activation.
We used Adam optimizer and set learning rate as 1e-3 and weight decay as 1e-5.
Batch size is 100.
Note that the models of \citet{maron2018invariant} are re-implemented by PyTorch in reference to the author's implementation in Tensorflow. 

In terms of an objective function, we modified mean squared error (MSE) loss based on Theorem \ref{represention}.
In a straightforward manner, equivariant tasks are regression task so that MSE loss is selected to decrease the gap of all elements between an output matrix and ground truth matrix: $\ell_{std}\colon \mathbb{R}^{n\times n} \times \mathbb{R}^{n\times n} \rightarrow \mathbb{R}$.
However, thanks to Theorem \ref{represention}, ReyNet does not require calculating whole elements but the gap of 1st row of 1st column element and 1st row of 2nd column element: $\ell_{corner}\colon \mathbb{R}^{1\times 2} \times \mathbb{R}^{1\times 2} \rightarrow \mathbb{R}$.
In this paper, we call the latter objective function as corner MSE loss.

\subsection{Results}
\begin{wraptable}[10]{r}[0mm]{50mm}
 \caption{Comparison of Objective Function}
 \label{tab:exp2}
 \centering
 \small
    \begin{tabular}{l|ll}
    \toprule
        $n$    & mse & corner mse \\ \midrule
        3 & 1.438e-4 & \textbf{1.249e-4} \\
        5 & \textbf{4.912e-5} & 9.375e-5 \\
        10 & 1.157e-4 & \textbf{6.487e-5} \\
        20 & 1.608e-4 & \textbf{8.537e-5} \\ \bottomrule
    \end{tabular}
\end{wraptable}
Table \ref{tab:exp1} shows the result of synthetic datasets, which is the average of MSE of five different seeds.
As a result, our proposed models outperform \citet{maron2018invariant} overall.
Interestingly, with regards to the \textit{power} task, FNN and ReyNet are worse than other models.
The function $F$ of \textit{power} task amplifies the output error to large value.
Simultaneously, when the number of parameters of the model is very large along with the size of $n$, the output of the model becomes sensitive.
Therefore, we considered that the error became large due to the combination with large neural networks and the \textit{power} function.

Moreover, we validated the effect of an objective function.
We trained 4r-ReyNets with standard MSE loss and Corner MSE loss respectively.
Table \ref{tab:exp2} shows the result of using each objective function.
Accordingly, Corner MSE loss achieved lower error than standard MSE loss except for $n=5$ case.
Therefore, using Corner MSE loss is a good choice for equivariant tasks.

\subsection{Generalization Error}
\begin{figure}[t]
    \centering
  \begin{minipage}[b]{0.45\hsize}
    \centering
    \includegraphics[width=0.95\columnwidth]{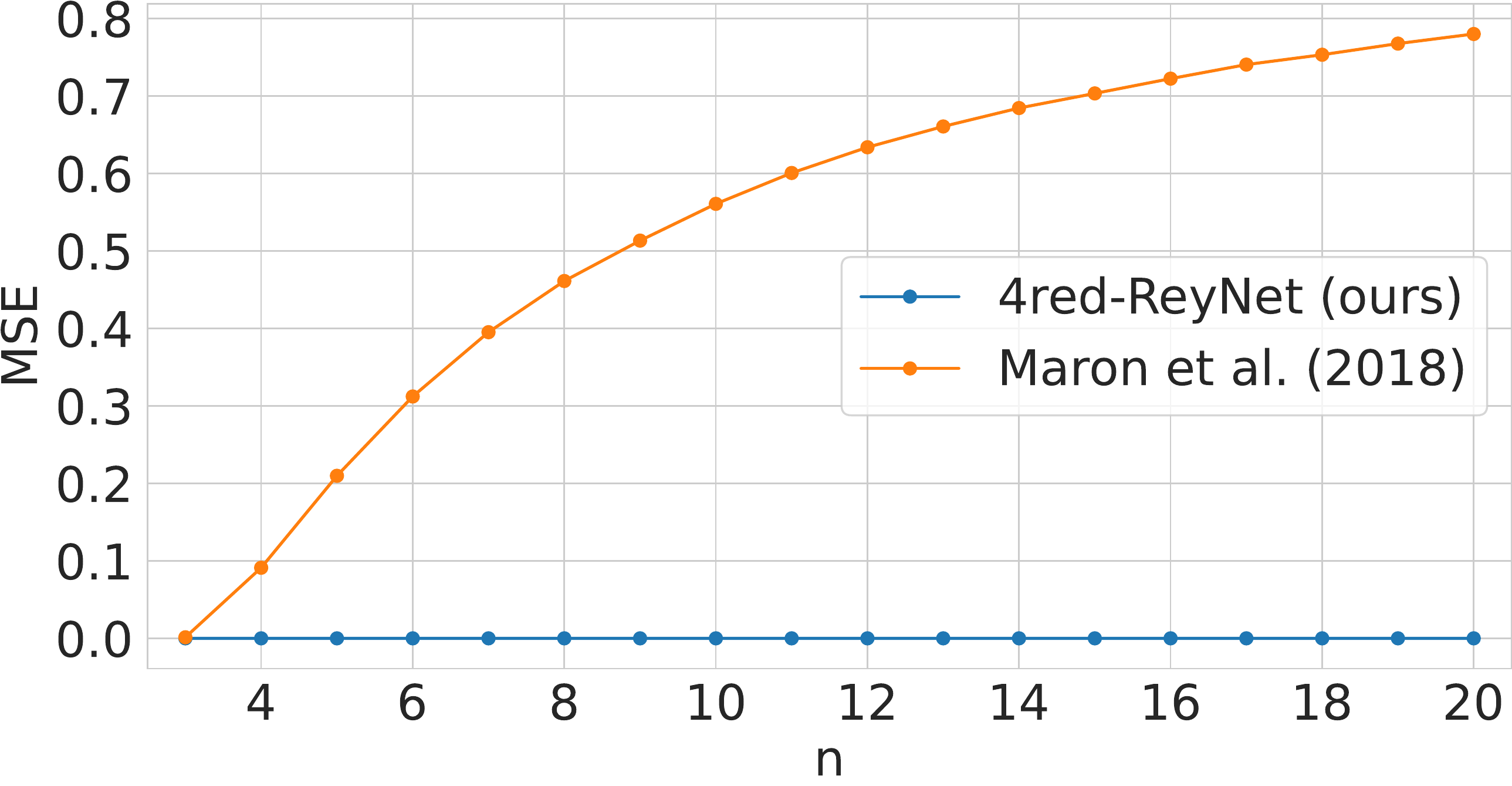}
    \subcaption{Symmetry}\label{gen-sym}
  \end{minipage}
  \begin{minipage}[b]{0.45\hsize}
    \centering
    \includegraphics[width=0.95\columnwidth]{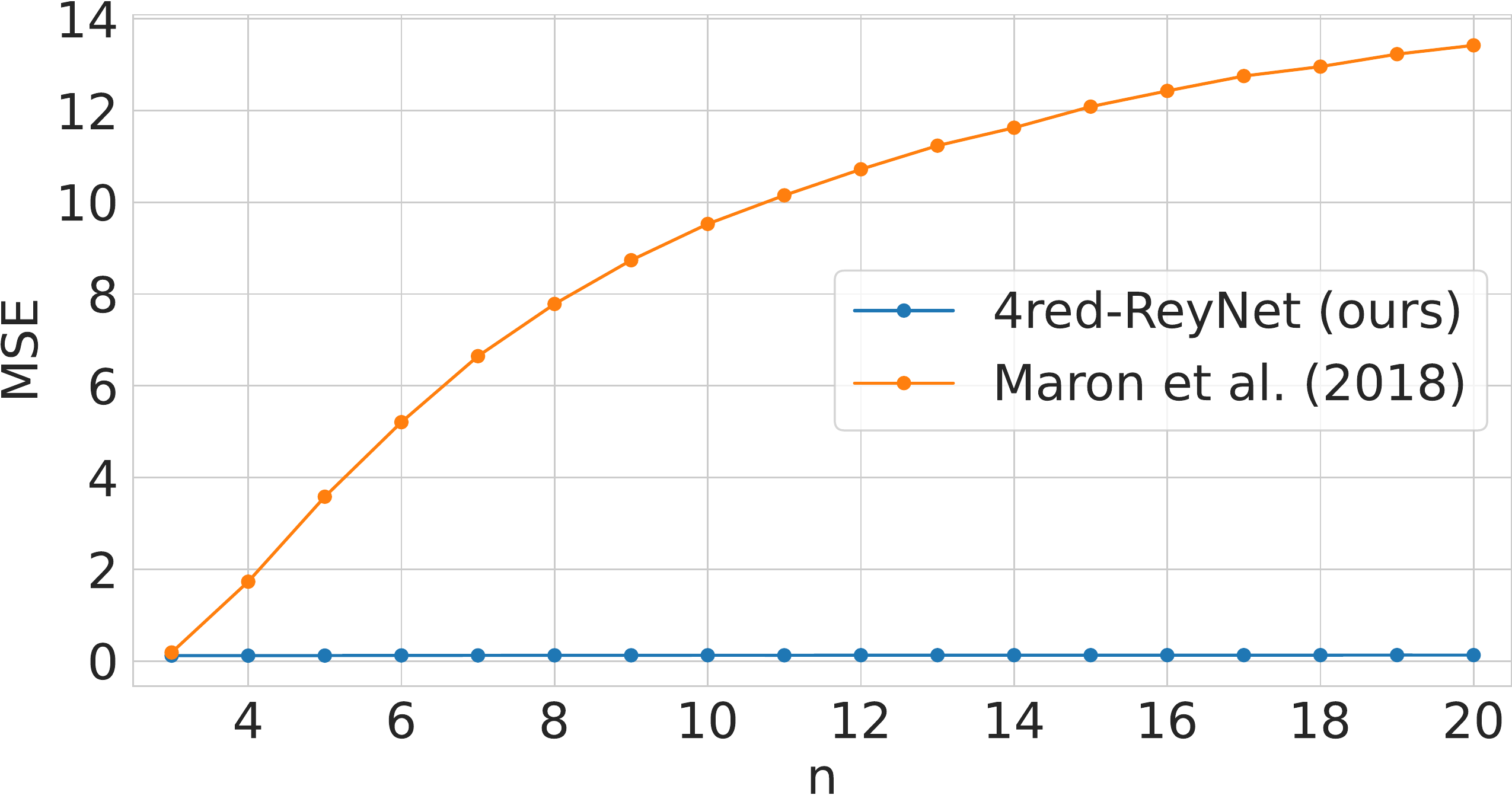}
    \subcaption{Power}\label{gen-diag}
  \end{minipage}
  \caption{Generalization Error}\label{fig:generalization}
\end{figure}
Notably, our 4red-ReyNet does not depend on the size of inputs $n$.
In order to test the generalization performance, we trained 4red-ReyNet with $n_{train}=3$ dataset and then validated the MSE on $n_{test}\in \{3, 4, \dots, 20\}$ datasets. 
The results are depicted in Figure \ref{fig:generalization}.
We can see that our 4-reduced ReyNet is generalized to the input size.
Note that with regard to invariant tasks, we confirmed the model is not generalized to the tasks as \citet{maron2018invariant} has reported.


\section{Conclusion}
We considered equivariant and invariant neural networks over higher order tensor spaces. 
The method of converting deep neural nets to equivariant neural nets using Reynolds operators had some difficulties in terms of computational complexity, which can be successfully avoided by using Reynolds designs.
Then, we constructed equivariant Reynolds networks (equivariant ReyNets) based on the Reynolds designs and proved the universality.
Similarly, we also introduced Reynolds designs induced by Reynolds dimension in the invariant case.
Then, we also constructed invariant Reynolds networks (invariant ReyNets) and proved the universality as well.
Furthermore, we showed that input variables of ReyNets can be reduced based on the Reynolds dimension.
In the section of numerical experiments,
we showed that equivariant and invariant ReyNets performs better than or comparable to existing models.
Moreover, we observed that ReyNets with a few input variables can generalize well to the cases with more input variables.

\newpage
\bibliography{sannai_229.bib}
\bibliographystyle{iclr2022_conference}

\newpage
\appendix

\section{Proof of (\ref{eq:representation})}
We provide the detailed calculation in (\ref{eq:representation}).
\begin{align*}
F(x) 
&=\sum_{\mathbf{u}\in\left[n\right]^{m}} f_{\mathbf{u}} \cdot \hat{e}_{\mathbf{u}, b}(x) \\
&=\sum_{1\leq D\leq m}\sum_{g \in H_{D}} \sum_{\textbf{T} \in \mathcal{T}_{m, D}} f_{g^{-1} \cdot \textbf{T}} \cdot \hat{e}_{g^{-1} \cdot \textbf{T}, b}(x) \\
&=\sum_{1\leq D\leq m}\sum_{\textbf{T} \in \mathcal{T}_{m, D}} \sum_{g \in H_{D}} f_{g^{-1} \cdot \textbf{T} } \cdot \hat{e}_{g^{-1} \cdot \textbf{T} , b}(x) \\
&=\sum_{1\leq D\leq m}\sum_{\textbf{T}  \in \mathcal{T}_{m, D}} \sum_{g \in H_{D}} f_{ \textbf{T} }\cdot \hat{e}_{g^{-1}\cdot \textbf{T} , b}(g\cdot x) \\
&=\sum_{1\leq D\leq m}\sum_{\textbf{T}  \in \mathcal{T}_{m, D}} \sum_{g \in H_{D}}g^{-1}\cdot (f_{\textbf{T} }\cdot \hat{e}_{\textbf{T} , b}(g\cdot x)) \\
&=\sum_{1\leq D\leq m}\sum_{\textbf{T}  \in \mathcal{T}_{m, D}} \sum_{g \in H_{D}}\frac{1}{|H_D|}g^{-1}\cdot (|H_D|f_{\textbf{T} }\cdot \hat{e}_{\textbf{T} , b}(g\cdot x)) \\
&=\sum_{1\leq D\leq m}\sum_{\textbf{T}  \in \mathcal{T}_{m, D}} \tau_{H_{D}}\left( |H_D|f_{\textbf{T} } \cdot \hat{e}_{\textbf{T} , b}\right)(x)\\
&=\sum_{1\leq D\leq m}\sum_{\textbf{T}  \in \mathcal{T}_{m, D}} \tau_{H_{D}}\left( F_{\textbf{T} } \cdot \hat{e}_{\textbf{T} , b}\right)(x),
\end{align*}
where the fourth equality follows from (\ref{eq:component}) and the last equality follows by putting $F_\textbf{T} :=|H_D|f_\textbf{T}$.
\hspace{\fill}$\blacksquare$

\section{Proof of Theorem \ref{thm:universality-equiv}}

By Theorem \ref{represention}, we have continuous maps $f_{\mathbf{T}}: \mathbb{R}^{n^l \times a} \rightarrow \mathbb{R}^b$ satisfying $F= \sum_{D=1}^m \sum_{\mathbf{T} \in \mathcal{T}_{m,D}} \tau_{H_D}(f_{\textbf{T}
}\circ \hat{\mathbf{e}}_{\mathbf{T}, b}),$ for standard Young tableaux $\mathbf{T} \in \mathcal{T}_{m,D}$. Since $K$ is a compact set and $S_n$ is a finite group, we may assume that $K$ is closed under $S_n$-action by taking $\bigcup_{g\in S_n} g\cdot K$.Then for any $\varepsilon$, we have an MLP $\mathcal{N
}_\textbf{T}$ which approximates $f_\textbf{T}$, namely $\left\|N_{T}-f_{T}\right\|_{K}<\varepsilon$ holds. Hence by the definition of our invariant model, we have
\begin{align*}
&\left\|F-\sum_{D=1}^{m} \sum_{\textbf{T} \in \mathcal{T}_{m, D}} \tau_{H_{D}}\left(\mathcal{N}_{T} \cdot \hat{e}_{\textbf{T}, b}\right)\right\|_{K}\\
&=\| \sum_{D=1}^{m} \sum_{T \in \mathcal{T}_{m,D}} \tau_{H_{D}}\left(f_{T} \cdot \hat{e}_{\textbf{T}, b}\right)-\sum_{D=1}^{m} \sum_{T \in \mathcal{T}_{m, D}} \tau_{H_D}\left(N_{T} \hat{e}_{\textbf{T}, b} \right)\|_K  \\ 
 &\leq \sum_{D=1}^{m} \sum_{T \in \mathcal{T}_{m, D}} \| \tau_{H_{D}}\left(f_{T} \cdot \hat{e}_{\textbf{T}, b}\right)-\tau_{H_{D}}\left(N_{T} \cdot \hat{e}_{T, b}\right) \|_{K} \\
&\leq \sum_{D=1}^{m} \sum_{T \in \mathcal{T}_{m, D}}\left\|\tau_{H_{D}}\left(\left(f_{T}-N_{T}\right) \cdot \hat{e}_{T, b}\right)\right\|_{K} \\
&\leq \sum_{D=1}^{m} \sum_{T \in \mathcal{T}_{m, D}} \| \sum_{g \in H_{D}} \frac{1}{\left|H_{D}\right|}\left(\left(f_{T}-N_{T}\right) \cdot \hat{e}_{g^{-1}\cdot T, b}(g \cdot -) \right\|_{K} \\
&\leq \sum_{D=1}^{m} \sum_{T \in \mathcal{T}_{m, D}} \sum_{g \in H_{D}} \frac{1}{\left|H_{D}\right|}\left\|\left(f_{T}-N_{T}\right) \cdot \hat{e}_{g^{-1} \cdot T, b} (g \cdot-)\right\|_K\\
&\leq \sum_{D=1}^{m} \sum_{T \in \mathcal{T}_{m, D}} \sum_{g \in H_{D}} \frac{1}{\left|H_{D}\right|}\left\|\left(f_{T}-N_{T}\right) \cdot \hat{e}_{g^{-1} \cdot T, b} (-)\right\|_K\\
&\leq \sum_{D=1}^{m} \sum_{T \in \mathcal{T}_{m, D}} \sum_{g \in H_{D}} \frac{1}{\left|H_{D}\right|}\left\|\left(f_{T}-N_{T}\right)\right\|_K \\
&\leq \sum_{D=1}^{m} \sum_{T \in \mathcal{T}_{m, D}} \sum_{g \in H_{D}} \frac{1}{\left|H_{D}\right|}\varepsilon \\
&\leq m |\mathcal{T}_{m,D}| \varepsilon.
\end{align*}
By replacing $\varepsilon$,  we obtain $\left\|F-\sum_{D=1}^{m} \sum_{\textbf{T} \in \mathcal{T}_{m, D}} \tau_{H_{D}}\left(\mathcal{N}_{T} \cdot \hat{e}_{\textbf{T}, b}\right)\right\|_{K}< \varepsilon$.
\hspace{\fill}$\blacksquare$

\section{Proof of Proposition \ref{inv design} and Theorem \ref{thm:universality-inv}} 
We use the following theorem.
\begin{thm}[Hilbert finiteness theorem (\cite{hilbert1890})]\label{Hilbert}
Let $G$ be a finite group or, more generally, a linearly reductive group. In this case, there is always a generator of G-invariant polynomials.
\end{thm}
Let $f: \mathbb{R}^{n^{\ell} \times a} \rightarrow \mathbb{R}^b$ be a continuous invariant function. 
By replacing $K$ with $\bigcup_{g \in G} g \cdot K$, we may assume that $K$ is closed under the action of $G$. 
Then by the Stone-Weierstrass theorem, there exists a polynomial $\hat{f}: \mathbb{R}^{n^{\ell} \times a} \rightarrow \mathbb{R}^b$  which approximate $f$ with arbitrary precision on $K$. Put $\tilde{f}= \gamma_G({\hat{f}})$, then
\begin{align*}
\left\|f(\textbf{x})-\gamma_{G}(\hat{f}(\textbf{x}))\right\|_{K} &=\frac{1}{\left|G\right|}\left\||G| f(\textbf{x})-\sum_{g \in G} f(g\cdot \textbf{x})\right\|_{K}\\
&\leq \frac{1}{\left|G\right|}\sum_{g \in G}\|f(\textbf{x})-\hat{f}(g \cdot \textbf{x})\|_{K} \\
&=\frac{1}{|G|} \sum_{g\in G}\left\|f\left(g\cdot \textbf{x}\right)-\hat{f}(g \cdot \textbf{x})\right\|_{K}\\
&\leq \frac{1}{\left|G\right|} \sum_{g \in G} \varepsilon=\varepsilon,
\end{align*}
where we used the property $f(\textbf{x}) = f(g\cdot \textbf{x})$  in the third equation. By Theorem \ref{Hilbert}, we have a generator of invariant polynomials $r_{1}, \ldots, r_s$. From the definition of generator, there exists a polynomial $P$ and $\tilde{f}$ can be written in the form $\tilde{f}\left(x_{1}, \ldots, x_{n^la}\right) = P(r_1\left(x_{1}, \ldots, x_{n^la}\right),..,r_s\left(x_{1}, \ldots, x_{n^la}\right))$. 
By the assumption of Reynolds dimension,  $r_{1}\left(x_{1}, \ldots, x_{n^la}\right), \ldots, r_s\left(x_{1}, \ldots, x_{n^la}\right)$ are written as $\gamma_{G}\left(h_{1}\left(x_{j_1}, \ldots, x_{j_d}\right)\right), \ldots, \gamma_{G}\left(h_{s}\left(x_{j_1}, \ldots, x_{j_d}\right)\right)$ for some polynomials  $h_{1}\left(x_{j_1}, \ldots, x_{j_d}\right), \ldots, h_{s}\left(x_{j_1}, \ldots, x_{j_d}\right)$ of $d$-variables. 

Since $H_d$ is a complete system of representative of $G/ \operatorname{Stab}([d]) $, 
we obtain the following decomposition:
\begin{align*}
    G 
    = \bigcup_{g\in [G/\operatorname{Stab}_{G}([d])]}g\cdot \operatorname{Stab}_{G}([d])
    = \bigcup_{g\in H_d}g\cdot \operatorname{Stab}_{G}([d]).
\end{align*}
Then, this induces the decomposition of Reynolds operators; 
 \begin{align*}
\gamma_{G}= \gamma_{H_d}\circ \gamma_{\operatorname{Stab}([d])} : \mathbb{R}\left[x_{1}, . ., x_{n^la}\right] \to \mathbb{R}\left[x_{1}, . ., x_{n^la}\right]^{\operatorname{Stab}([d])} \rightarrow \mathbb{R}\left[x_{1}, \ldots, x_{n^la}\right]^{G},
 \end{align*}
 where $\mathbb{R}\left[x_{1}, . ., x_{n^la}\right], \mathbb{R}\left[x_{1}, . ., x_{n^la}\right]^{\operatorname{Stab}([d])} , \mathbb{R}\left[x_{1}, \ldots, x_{n^la}\right]^{G}$ are  the set of polynomials, $\operatorname{Stab}([d])$-invariant polynomials, invariant polynomials, respectively.
This implies 
\begin{align*}
r_i=\gamma_{G}\left(h_{i}\left(x_{j_1}, \ldots, x_{j_d}\right)\right)=\gamma_{H_d}\left( \gamma_{\operatorname{Stab}([d])}(h_{i}\left(x_{j_1}, \ldots, x_{j_d})\right)\right)=\gamma_{H_d}\left(h_{i}\left(x_{j_1}, \ldots, x_{j_d}\right)\right).
\end{align*}
Here the last inequality follows from the fact that $h_{i}\left(x_{j_1}, \ldots, x_{j_d}\right)$ is a $\operatorname{Stab}_{G}([d])$-invariant polynomial. On the other hand, note that the invariant Reynolds operator is equal to the composition of the equivariant Reynolds operator and the orbit sum; $\gamma_G =\Sigma \circ \tau_G$.
For the vector valued function $h=(h_1,..,h_s)$, by the universal approximation theorem of fully connected neural nets, we can take a fully connected neural net $\mathcal{Q}$ with which 
$\|\mathcal{Q}- h\|_K < \varepsilon$ holds.
Let $\mathcal{N}$ be a fully connected neural net that approximates $P$ above; $\|\mathcal{N}- P\|_K < \varepsilon$. Then 
\begin{align*}
\mathcal{N}  \circ \Sigma \circ  \tau_{H_d} \circ Q(x_1,..,x_n)
&\approx \mathcal{N} \circ \Sigma \circ \tau_{H_d}  (h_1(x_1,..,x_n),..,h_s(x_1,..,x_n)) \\
&\approx  \mathcal{N}(\gamma_G(h_1)(x_1,..,x_n),..,\gamma_G(h_s)(x_1,..,x_n)) \\
&=  \mathcal{N}(r_1(x_1,..,x_n),..,r_s(x_1,..,x_n))\\
&\approx  P(r_1(x_1,..,x_n),..,r_s(x_1,..,x_n))\\
&= \tilde{f}(x_1,..,x_n)\\
&\approx f(x_1,..,x_n).
\end{align*}
\hspace{\fill}$\blacksquare$

\section{Generalization Error}
Figure \ref{fig:generalization2} shows the detail version of only 4red-ReyNets at Figure \ref{fig:generalization}.
In figure 3 our results seems an almost horizontal straight line, but this figure shows that the MSE also slightly increases as $n$ increases.
\begin{figure}[h]
    \centering
  \begin{minipage}[b]{0.45\hsize}
    \centering
    \includegraphics[width=0.95\columnwidth]{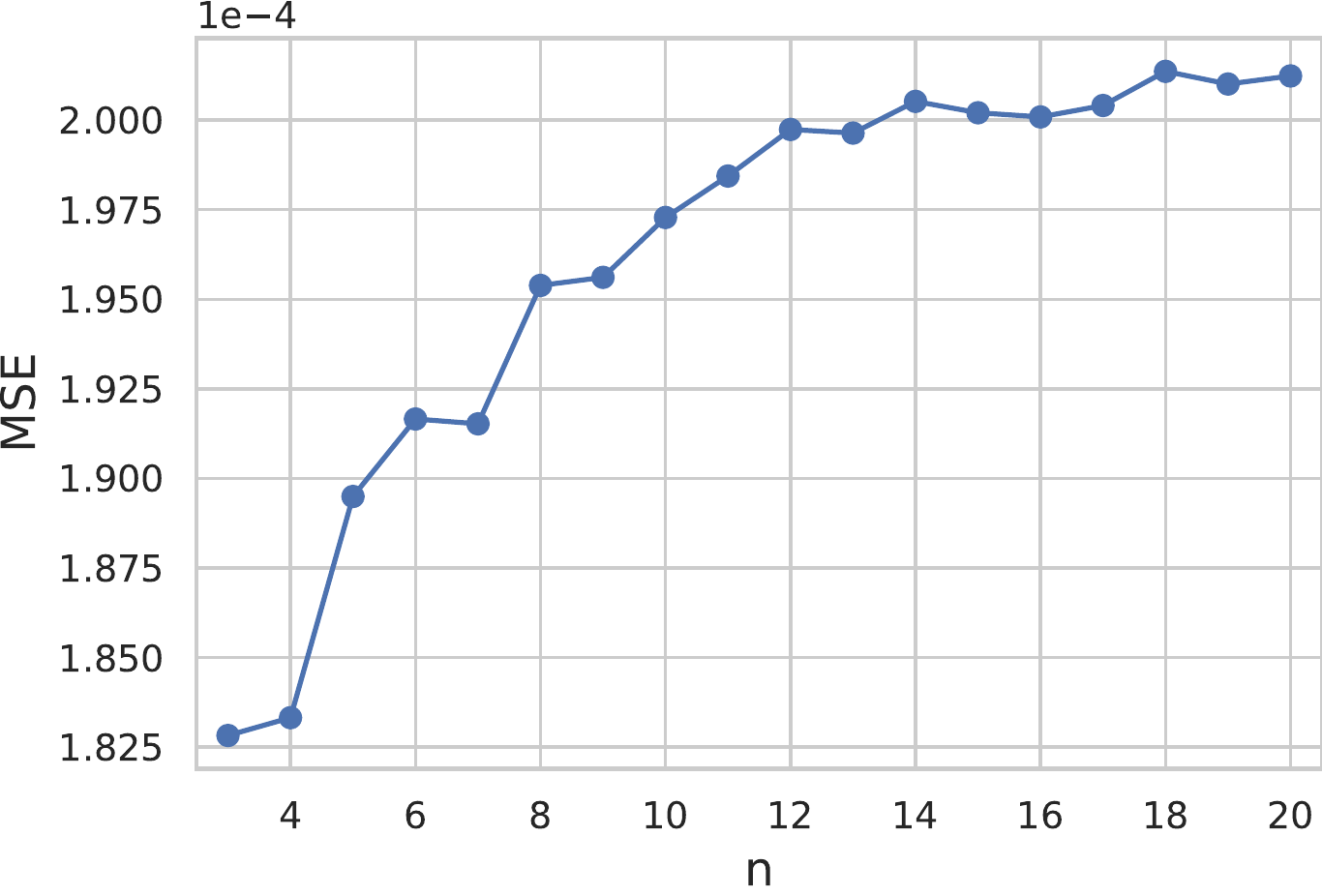}
    \subcaption{Symmetry}\label{gen-sym2}
  \end{minipage}
  \begin{minipage}[b]{0.45\hsize}
    \centering
    \includegraphics[width=0.95\columnwidth]{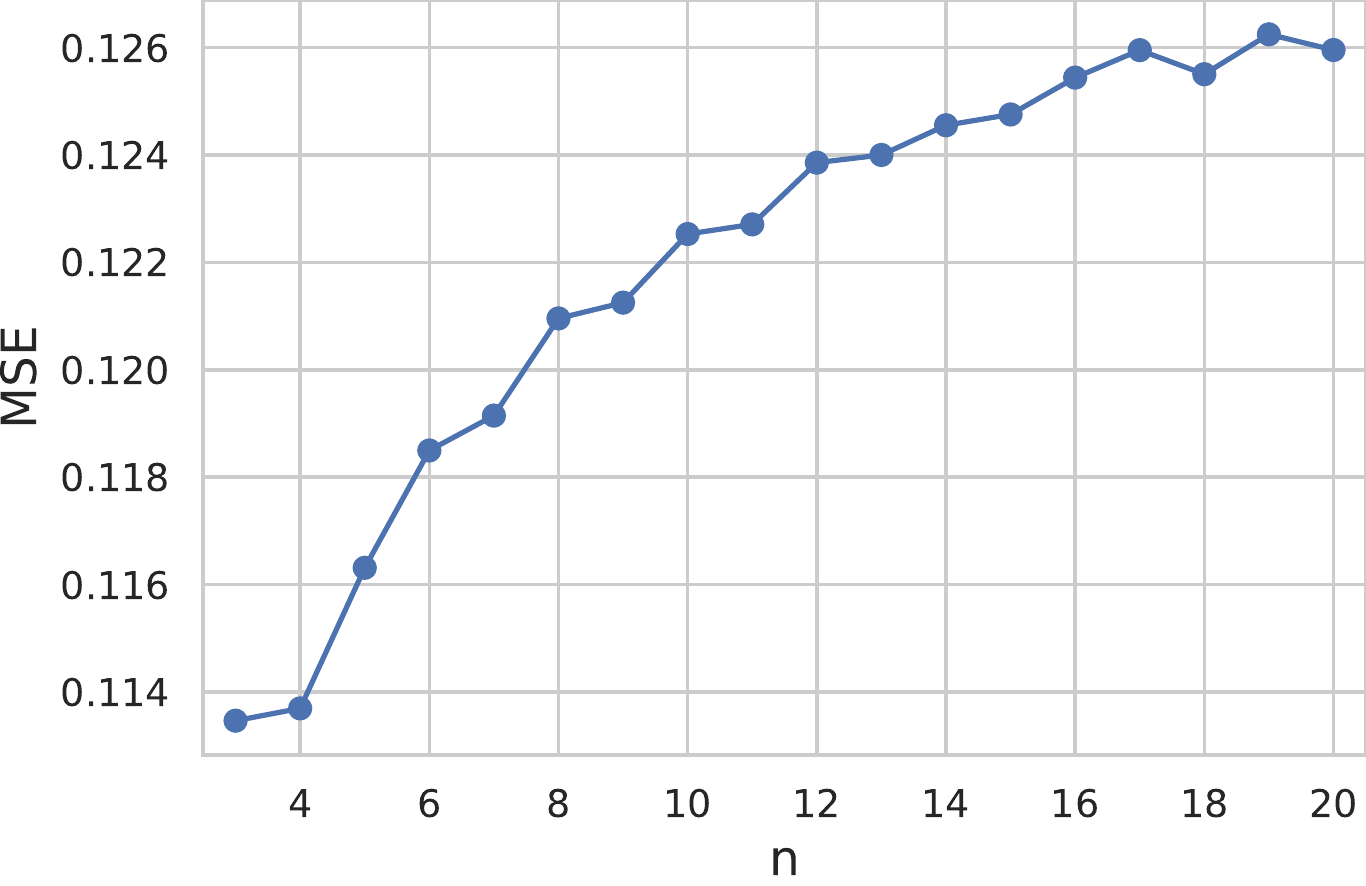}
    \subcaption{Power}\label{gen-diag2}
  \end{minipage}
  \caption{Generalization Error}\label{fig:generalization2}
\end{figure}

\end{document}